\documentclass[letterpaper, 10 pt, journal, twoside]{IEEEtran}
\IEEEoverridecommandlockouts 
\usepackage[colorlinks=true,linkcolor=black,anchorcolor=black,citecolor=black,filecolor=black,menucolor=black,runcolor=black,urlcolor=black]{hyperref}
\usepackage{url}
\usepackage[skip=5pt,font=scriptsize]{caption}
\usepackage{graphicx}
\usepackage{subfigure}
\usepackage{etoolbox}
\usepackage{scalerel}
\usepackage{siunitx}
\usepackage{amsmath,textcomp}

\usepackage{amssymb}
\usepackage{amsthm}

\usepackage[ruled,vlined,linesnumbered]{algorithm2e}
\SetKwInput{KwInput}{Input}
\SetKwInput{KwInit}{Initialize}
\SetKwInput{KwOutput}{Output}
\SetKwInOut{KwInOut}{Parameter}
\usepackage{bbm}
\usepackage{booktabs}
\usepackage{url}
\usepackage{paralist}
\usepackage{color}
\usepackage{mathtools}
\usepackage{xspace}
\usepackage{verbatim}
\usepackage{epstopdf}
\usepackage{listings}
\usepackage{lstautogobble}
\usepackage[utf8]{inputenc}
\usepackage{calrsfs}
\usepackage{rotating,multirow}

\usepackage[utf8]{inputenc}

\usepackage{enumitem}
\usepackage{hyperref}

\usepackage{titlesec}
\titlespacing{\section}{0pt}{0.4ex}{0.4ex}
\titlespacing{\subsection}{0pt}{0.2ex}{0.2ex}
\titlespacing{\subsubsection}{0pt}{0.1ex}{0.1ex}

\allowdisplaybreaks

\newtheorem{definition}{Definition}

\newtheorem{theorem}{Theorem}
\newtheorem{remark}{Remark}
\newtheorem{lemma}{Lemma} 
\newtheorem{assumption}{Assumption}

\renewcommand{\footnotesize}{\fontsize{8pt}{11pt}\selectfont}

\DeclareMathAlphabet{\mathcal}{OMS}{cmsy}{m}{n}
\newcommand{\regtext}[1]{\mathrm{\textnormal{#1}}}

\newcommand{\R}{{\mathbb{R}}}

\newcommand{\N}{{\mathbb{N}}}

\DeclarePairedDelimiter{\norm}{\lVert}{\rVert}

\newcommand{\opt}{^\star}

\newcommand{\vc}[1]{{\mathbf{#1}}}

\newcommand{\zeros}{\vc{0}}

\newcommand{\zono}[1]{\mathcal{Z}\!\left(#1\right)}

\newcommand{\ctr}{\vc{c}}

\newcommand{\Gen}{\vc{G}}

\newcommand{\idx}[1]{^{(#1)}}
\newcommand{\arridx}[2]{{\left(#1\right)}_{#2}}
\newcommand{\idxi}{\idx{i}}

\newcommand{\lbl}[1]{_{\regtext{#1}}}
\newcommand{\obs}{\lbl{obs}}
\newcommand{\goal}{\lbl{goal}}

\newcommand{\total}{\lbl{total}}
\newcommand{\brk}{\lbl{brk}}

\newcommand{\tbold}[1]{#1}

\newcommand{\nrl}{{n\lbl{RL}}}

\newcommand{\niter}{{n\lbl{iter}}}

\newcommand{\nplan}{{n\lbl{plan}}}

\newcommand{\ntraj}{{q}}

\newcommand{\vcstate}{x}
\newcommand{\vcoutput}{y}
\newcommand{\action}{u}
\newcommand{\plan}{\vc{p}}
\newcommand{\noise}{w}
\newcommand{\measnoise}{v}
\newcommand{\statespace}{Y}
\newcommand{\actionspace}{\mathcal{Z}_{u,t}}
\newcommand{\noisespace}{\mathcal{Z}_w}
\newcommand{\measnoisespace}{\mathcal{Z}_v}
\newcommand{\statedata}{Y}
\newcommand{\actiondata}{U_-}

\newcommand{\RLoutput}{\hat{\vcoutput}}
\newcommand{\RLaction}{\action}
\newcommand{\policy}{\pi}

\newcommand{\rewfunc}{\rho}
\newcommand{\rew}{r}
\newcommand{\totaltime}{{n\lbl{total}}}

\newcommand{\reachset}{R}
\newcommand{\reachapprox}{{\hat{\reachset}}}
\newcommand{\enforcefunc}[1]{\regtext{enforce\_safety}\!\left(#1\right)}

\begin{document}

    
\title{Safe Reinforcement Learning using Data-Driven Predictive Control}

\author{Mahmoud Selim$^{1}$, Amr Alanwar$^{2}$, M. Watheq El-Kharashi$^{1}$, Hazem M. Abbas$^{1}$, Karl H. Johansson$^{3}$ \vspace{-5mm}
\thanks{$^{1}$Ain Shams University, Cairo, Egypt.
$^{2}$Jacobs University, Bremen, Germany. 
$^{3}$KTH Royal Institute of Technology, Stockholm, Sweden.
Corresponding author: \texttt{mahmoud.selim@eng.asu.edu.eg}.}
}


\maketitle

\begin{abstract}
Reinforcement learning (RL) algorithms can achieve state-of-the-art performance in decision-making and continuous control tasks. However, applying RL algorithms on safety-critical systems still needs to be well justified due to the exploration nature of many RL algorithms, especially when the model of the robot and the environment are unknown. To address this challenge, we propose a data-driven safety layer that acts as a filter for unsafe actions. The safety layer uses a data-driven predictive controller to enforce safety guarantees for RL policies during training and after deployment. The RL agent proposes an action that is verified by computing the data-driven reachability analysis. If there is an intersection between the reachable set of the robot using the proposed action, we call the data-driven predictive controller to find the closest safe action to the proposed unsafe action. The safety layer penalizes the RL agent if the proposed action is unsafe and replaces it with the closest safe one. In the simulation, we show that our method outperforms state-of-the-art safe RL methods on the robotics navigation problem for a Turtlebot 3 in Gazebo and a quadrotor in Unreal Engine 4 (UE4). 
\end{abstract}
\section{Introduction}

Safety of Cyber-Physical Systems (CPSs) has always been of great interest to researchers \cite{6061910}. This imposes a great challenge for reinforcement learning controllers due to the exploration nature of many reinforcement learning algorithms \cite{book:sutton1998introduction}. This is especially true when the environment and/or system models are both time-varying or subject to noise \cite{mihatsch2002risk}. Recently, many techniques have been proposed to mitigate the safety problems of RL algorithms. An RL agent is considered to be safe if it meets an ergodicity requirement that it can reach each state it visits from any other state it visits, allowing for reversible errors \cite{safety_in_mdps, garcia2015comprehensive}. To justify the widespread deployment of RL controllers, we need to ensure the safety of RL algorithms during training and after deployment \cite{ai_safety}. In this work, we propose a safety layer to address the safety challenges of RL algorithms in a totally data-driven way when the robot model is unknown. We utilize data-driven predictive control \cite{datadriven_predictive_control} to impose strict safety constraints on RL agents. This is considered a step towards making data-driven methods more convenient and practical for safety-critical systems.


\subsection{Related Work} \label{sec:related}

Safe RL aims to maximize an objective function while respecting safety constraints. Many methods have been previously proposed to address the problem of safety for RL agents. One popular approach imposes constraints on the expected return \cite{achiam2017constrained, reward_constr_policy}. Other include risk measures \cite{risk_beyond, risk_consideration,Geibel_2005,Shen_2014} %
or impose constraints on the Markov Decision Process (MDP) \cite{cmdp_constraints} or avoid regions in the MDP where the safety constraints may be violated \cite{turchetta2016safe, eysenbach2017leave}. We note that these approaches are different from just penalizing the agent with large negative rewards. This type of penalty can lead to undesired behavior from the agent and may cause the training to be unstable. 

Another major approach to safe RL modifies the exploration process to filter out unsafe actions taken by the RL agent. This can be achieved by two methods. The first one uses a risk-directed exploration, and the second depends on the use of external knowledge \cite{garcia2015comprehensive}. Risk-directed exploration usually uses a risk-measure to determine the probability of selecting different actions during the exploration process \cite{garcia2015comprehensive, smart_exploration_in_rl}. 
External knowledge, on the other hand, can be done by providing a set of demonstrations \cite{abbeel2005exploration}, or the use of control theoretic approaches. Control theoretic approaches usually require a partial or a complete system model to predict whether the system will reach an unsafe state and prevent it. For example, one can utilize control barrier functions and control Lyapunov functions to guarantee the system safety or stability \cite{molnar2021model, liu2022robot, robustCBFarticle, cbfforunmodeleddynamics}. Others use reachability-based methods to calculate the reachable sets of the system. Safety is then achieved when the reachable set of the system lies within the safe region. However, if they don't, action correction techniques can be used, such as sampling to find a safe action or projecting the RL action into the safe set. This can be done efficiently in case of parameterized reachability as in \cite{conf:NiklasRL}.

One final remark, our proposed method is closely related to \cite{dalal2018safe,9833266}. They both use data collected offline to enforce safety constraints on the system. However, the former uses a first-order approximation shown to be quite conservative. The latter gives no regard to how much the adjusted action is close to the RL agents' action of choice. The proposed approach, however, tries to be as little invasive as possible to the RL agents' action of choice, similar to \cite{conf:melanieRL}.

Trajectory or path planning is considered a fundamental problem in robotics. It aims to find the shortest and most obstacle-free path from the start to the goal state. The path can be a set of states (position and orientation) or waypoints. Many methods have been proposed to address the planning problems in robotics. They are mainly classified into two categories: classical approaches and learning-based approaches. Classical approaches can be further divided into grid-based search algorithms and sampling-based search algorithms. Grid-based search algorithms usually find a path using a minimum travel cost like Dijkstra’s algorithm and A$^*$. Grid-based algorithms are guaranteed to find the shortest path given enough computation time and resources. However, in most real-life scenarios, finding a reasonable path or a trajectory in an acceptable computation time is preferable. Sampling-based algorithms take this approach. That is, they sample the configuration space to reach a reasonable solution. Sampling algorithms such as RRT$^*$ \cite[Ch. 5]{lavalle2006planning} are probabilistically complete and are suitable for both low and high-dimensional search spaces. Recent approaches to path planning rely heavily on learning-based approaches. Many Deep Reinforcement Learning (DRL) techniques have been developed to address the planning problem, especially in uncertain environments. Some of these approaches depend on model-based reinforcement learning for planning \cite{mbrl_planning}. Others utilize hierarchical reinforcement learning for planning tasks \cite{hrl_planning, hrl_planning2}. A clear advantage of these methods is that, given enough training samples, they can be optimized while remaining computationally efficient.



\textbf{Limitations.}
We assume a discrete-time setting and model the system as a linear model. We leave the continuous-time setting in reachability analysis \cite{conf:modelBasedReachRL}), non-linear system models, and the amount of data required to guarantee safety for future work.

\textbf{Contributions.}
The main contributions of this paper can be summarized as follows:
\begin{enumerate}[leftmargin=15pt,noitemsep,partopsep=0pt,topsep=0pt,parsep=0pt]
\itemsep0em
    
    \item We propose a safety filter using data-driven predictive control to enforce safety constraints while being as little invasive as possible to the RL agents' choice of action when the model of the robot is unknown.
    
    \item We demonstrate the proposed method for the motion planning problem. We show through experiments that the proposed method is, indeed, able to guarantee safety. We also compare the proposed method against a baseline RL agent, Safe Exploration in Continuous Action Spaces (SECAS) \cite{dalal2018safe}, and Safe Advantage-based Intervention for Learning policies with Reinforcement (SAILR) \cite{wagener2021safe}.

\end{enumerate}

Next, in Section \ref{sec:pb}, we provide the problem statement and formulate the safe RL problem. 
Section \ref{sec:solu} discusses the proposed method, and Section \ref{sec:eval} goes through the evaluation and discusses the result.
Finally, Section \ref{sec:con} presents concluding remarks and discusses future work.




\section{Preliminaries and Problem Formulation} \label{sec:pb}

We now define the notation used throughout the paper, set representations for our reachability analysis, the system and its reachable sets, and our assumptions about data and noise.

\subsection{Notation and Set Representations}\label{subsec:notation}

Natural numbers are denoted by $\N$, and $n$-dimensional real numbers by $\R^n$. 
The element at row $i$ and column $j$ of a matrix $A$ is denoted by 
$(A)_{i,j}$, the column $j$ of $A$ by $\arridx{A}{:,j}$. 
The pseudoinverse of a matrix $A$ is denoted by $A^\dagger$.  
For simplicity, the Minkowski sum, defined by $\mathcal{Z}_1 \oplus \mathcal{Z}_2  
    = \{ z_1 + z_2 \,|\, z_1 \in \mathcal{Z}_1, z_2 \in \mathcal{Z}_2 \}$, is denoted by $+$ instead of $\oplus$ as the type can be determined from the context. Similarly, we use  $\mathcal{Z}_1 - \mathcal{Z}_2$ to denote $\mathcal{Z}_1 + -1 \mathcal{Z}_2$. The Cartesian product is given by $A \times B = \{\begin{bmatrix}a^\top & b^\top\end{bmatrix}^\top \ |\ a \in A, b \in B\}$.

Reachable sets are represented using zonotopes.
\begin{definition}(\textbf{Zonotope} \cite{conf:zono1998}) \label{def:zonotopes} 
Given a center $c_\mathcal{Z} \in \mathbb{R}^n$ and $\gamma_\mathcal{Z} \in \mathbb{N}$ generator vectors in a generator matrix $G_\mathcal{Z}=[g_\mathcal{Z}^{(1)},\dots,g_\mathcal{Z}^{(\gamma_\mathcal{Z})}] \in \mathbb{R}^{n \times \gamma_\mathcal{Z}}$, a zonotope is the set
\begin{equation*}
	\mathcal{Z} = \Big\{ x \in \mathbb{R}^n \; \Big| \; x = c_\mathcal{Z} + \sum_{i=1}^{\gamma_\mathcal{Z}} \beta^{(i)} \, g_\mathcal{Z}^{(i)} \, ,
	-1 \leq \beta^{(i)} \leq 1 \Big\} \; ,
\end{equation*}
where $\beta^{(i)}$ is called the zonotope's factor. We use the shorthand notation $\mathcal{Z} = \zono{c_\mathcal{Z},G_\mathcal{Z}}$ for a zonotope. 
\end{definition}

Given a linear map $L$, one can show that $L \mathcal{Z}  = \zono{L c_\mathcal{Z}, L G_\mathcal{Z}}$.
Given two zonotopes $\mathcal{Z}_1=\langle c_{\mathcal{Z}_1},G_{\mathcal{Z}_1} \rangle$ and $\mathcal{Z}_{\mathcal{Z}_2}=\langle c_{\mathcal{Z}_2},G_{\mathcal{Z}_2} \rangle$, their Minkowski sum is computed by 
\begin{equation}
     \mathcal{Z}_1 + \mathcal{Z}_2 = \Big\langle c_{\mathcal{Z}_1} + c_{\mathcal{Z}_2}, [G_{\mathcal{Z}_1}, G_{\mathcal{Z}_2}]\Big\rangle.
     \label{eq:minkowski}
\end{equation}
We compute the Cartesian product of two zonotopes $\mathcal{Z}_1 $ and $\mathcal{Z}_2$ by 
\begin{align}\label{eq:cart}
\mathcal{Z}_1 \times \mathcal{Z}_2 
= {\Big\langle} \begin{bmatrix} c_{\mathcal{Z}_1} \\ c_{\mathcal{Z}_2} \end{bmatrix}{,} \begin{bmatrix} G_{\mathcal{Z}_1} & 0 \\ 0 & G_{\mathcal{Z}_2} \end{bmatrix}\!\! \Big\rangle.
\end{align}

We also make use of matrix zonotopes to represent families of matrices:

\begin{definition}\label{def:matzonotopes}(\textbf{Matrix zonotope} \cite[p. 52]{conf:thesisalthoff})  
Given a center matrix $C_{\mathcal{M}} \in \mathbb{R}^{n\times T}$ and $\gamma_\mathcal{M} \in \mathbb{N}$ generator matrices $\tilde{G}_{\mathcal{M}}=[G_{\mathcal{M}}^{(1)},\dots,G_{\mathcal{M}}^{(\gamma_\mathcal{M})}] \in \mathbb{R}^{n \times (T  \gamma_\mathcal{M})}$, a matrix zonotope is defined as
\begin{equation}
	\mathcal{M} = \Big\{ X \in \mathbb{R}^{n\times T} \; \Big| \; X {=} C_{\mathcal{M}} + \sum_{i=1}^{\gamma_\mathcal{M}} \beta^{(i)} \, G_{\mathcal{M}}^{(i)} \, ,
	-1 \leq \beta^{(i)} \leq 1 \Big\} \; .
\end{equation}
\end{definition}

\begin{definition}\label{def:intmatrix}(\textbf{Interval Matrix} \cite[p.42]{conf:thesisalthoff})  
An interval matrix $\mathcal{I}$ specifies the interval of all possible values for each matrix element between the left limit $\underline{I}$ and right limit $\overline{I}$:

\begin{equation}
	\mathcal{I} = [\underline{I}, \overline{I}], \qquad \underline{I},\overline{I} \in \R^{r\times c}
\end{equation}

\end{definition}

\subsection{System Dynamics and Reachable Sets}
We assume a discrete time linear system model:
\begin{align}\label{eq:sys}
    \begin{split}
        \vcstate(t+1) &=  \textit{A} \vcstate(t) + \textit{B} \action(t) + \noise(t). \\ 
        \vcoutput(t) &= \textit{C} \vcstate(t) + \measnoise(t).
    \end{split}
\end{align}
with the system matrices $A \in \mathbb{R}^{n \times n}$ and $B \in \mathbb{R}^{n \times m}$, $C \in \mathbb{R}^{p \times n}$, 
state $x(t) \in \mathbb{R}^{n}$, input $u(t) \in \mathbb{R}^{m}$, process noise $w(t)\in \mathbb{R}^{n}$, and measurement noise $v(t)\in \mathbb{R}^{n}$. We assume that the states of the system are measurable and compact, i.e., the system output matrix is given by $C=I$, and thus the measured output is $y(t) \in \mathbb{R}^n$. Given that $C=I$, we call the $y(t)$ the state which is corrupted by extra noise $v(t)$.
The input and output constraints are given by 
\begin{equation}
    \begin{split}
    u(t) \in \mathcal{U}_t \subset \mathbb{R}^{m}, \\
    y(t) \in \mathcal{Y}_t \subset \mathbb{R}^n. 
    \end{split}
    \label{eq:sys_con}
\end{equation}

Reachability analysis computes the reachable set of a system, which is the set of states $y(t)$ which can be reached given a set of uncertain initial states $\mathcal{Y}_0 \subset \mathbb{R}^n$ and a set of possible inputs $\actionspace$ at each time step $t \in \N$.
More formally, we define the reachable set at time $N \in \N$ as follows:
\begin{definition} 
The reachable set $\reachset_{t}$ after $N$ time steps, subject to a sequence of inputs $\action(t) \in \actionspace$, noise $\noise(t) \in \noisespace$, measurement noise $\measnoise(t) \in \measnoisespace$, $\forall\ t \in \{ 0, \dots, N-1\}$, and initial set $\statespace_{0} \in  \mathbb{R}^n$, is the set
\begin{align}
        \reachset_{N} = \big\{& y(N) \in \mathbb{R}^n  \big|
            \vcstate(t+1) {=} A \vcstate(t)+ B\action(t) + \noise(t),\nonumber\\ 
            &y(t) = x(t) +v(t), \vcstate_{0} \in \statespace_{0}, \action(t) \in \actionspace,v(t) \in \mathcal{Z}_v, \nonumber\\&  \regtext{and}\ \noise(t) \in \noisespace, \forall\ t = 0,\cdots,N-1\big\}.\label{eq:reachable_set_k}
\end{align}
\end{definition}

In this work, we assume unknown system dynamics. However, we still seek to over-approximate the reachable sets of the system using data-driven reachability \cite{alanwar2020data_conf}. This way, we show that we are able to enforce safety constraints on a black-box system leveraging only data-driven methods.

To enable safety guarantees, we leverage the notion of failsafe maneuvers from mobile robotics \cite{kousik2020bridging}. 
Note that many real robots have a braking safety controller available, similar to the notion of an invariant set \cite{LewEtAl2021,conf:modelBasedReachRL}.  %
Also, failsafe maneuvers exist even when a robot cannot remain stationary, such as loiter circles for aircraft \cite{fridovich2019safely}. 

We do not assume a distribution for the noise. However, we require that the noise obeys the following assumptions for numerical tractability and robustness guarantees.

\begin{assumption}\label{assumption:noise_zonotope}
We assume that $\noise(t)$ is bounded by 
 \emph{noise zonotope} $\noisespace = \zono{\ctr_\noise,\Gen_\noise}$. 
\end{assumption}

\begin{assumption}\label{assumption:meas_noise_zonotope}
We assume that $\measnoise(t)$ is bounded by a zonotope $\measnoisespace = \zono{\ctr_\measnoise,\Gen_\measnoise}$. Furthermore, we assume that the one-step propagation $Av(t)$ is bounded by a zonotope $Av(t) \in \mathcal{Z}_{Av}$ for all time steps. 
\end{assumption}



In this work, unsafe regions of state space, or \emph{obstacles} are denoted by as $\statespace\obs \subset \R^n$, and since the focus of this work is not motion forecasting of other actors in the environment, obstacles are assumed to be static and not to move (yet they change from one episode to another). We emphasize, however, that this work can be extended to handle dynamic environments by using frameworks that use the reachability of other agents \cite{leung2020infusing,vaskov2019not}.

\subsection{Safe RL Problem Formulation}

We formulate the safe RL problem as follows.
Let $\statespace\goal \subset \R^n$ denote a goal the agent aims to reach while avoiding the obstacles (unsafe sets) $\statespace\obs \subset \R^n$. 
Let $\totaltime \in \N$ denote a total number of time steps in the training procedure, and $\niter$ denote the number of steps per episode.
At time $t$, let $\RLoutput(t) \in \R^\nrl$ denote the state of the RL agent (which contains the state $y(t)$ of the robot, plus additional information such as sensor measurements and previous actions).
Let $\RLaction(t)$ denote the action chosen by the RL agent at time $t$.
Let
\begin{align}\label{eq:reward_defn}
    \rew: (\RLoutput(t),\statespace\goal,\RLaction(t)) \mapsto \R 
\end{align}
denote a real-valued reward function.
We seek to learn a policy
\begin{align}
    \policy\opt: \RLoutput(t) \mapsto \action(t)
\end{align}
which maximizes the cumulative reward.
Furthermore, we require that, if we roll out the trajectory given the policy $\pi$ from any initial condition in the set $\statespace_0$, then the trajectory is safe.

\setlength{\textfloatsep}{0.45cm}
\setlength{\floatsep}{0.45cm}
\begin{algorithm}[t]
\caption{Safe Reinforcement Learning}
\label{alg:rlwithddps}
Initialize the RL agent with a random policy $\policy_\theta$, empty replay buffer $B$, number of steps per episode $\niter$, and a safe plan $\plan_0$

\For{each episode}{
    \textbf{initialize} RL task.
    
    $\RLoutput(1) \leftarrow $ observe initial environment state. \label{ln:observe1}
    
    \For{$k = 1:\niter$}{
        
    \tbold{ $\action(k) \leftarrow$ sample RL policy} \label{ln:sample_rl}
     
         $\reachapprox_k \leftarrow \zono{\vcoutput(k),\zeros}$ // init. reachable set \label{ln:initZono}
        
         $\plan_k$  $\leftarrow \enforcefunc{\reachapprox_k,\action(k),\vcoutput(k)}$ // Alg.~\ref{alg:enforce_safety} \label{ln:enforce_safety}
         
        %
       \If{$\plan_k == \emptyset$ \label{ln:if_unsafe}}{
           // if there is no found safe plan
            
            execute failsafe maneuver; continue\label{ln:catch_unsafe} 
        }
        
        $\action(k) \leftarrow$
        get first (safe) action from $\plan_k$ \label{ln:get_safe_action}

         $r_k \leftarrow \rewfunc(\RLoutput(k),\action(k))$ // get reward \label{ln:getReward}
        
         $\RLoutput(k + 1) \leftarrow $ observe new state \label{ln:observe2}
        
         \textbf{add} $(\RLoutput(k), \action(k), r_k, \RLoutput(k + 1))$ to $B$ \label{ln:addToBuffer}
        
         \textbf{train} the RL agent $\policy_\theta$ using a batch from $B$ \label{ln:train_rl_agent}
    }
}
\end{algorithm}

\section{Safe Reinforcement Learning}\label{sec:solu}

We propose a new framework to enforce safety constraints on the RL agents without previous knowledge of the environment or the robot model apriori. The idea is to introduce a safety layer that filters out any unsafe action choice taken by the RL agent in a totally data-driven way while being as little invasive as possible for the agent's choice of action. Since the framework uses data-driven methods to enforce safety, the proposed method can handle time-varying systems by updating the data matrices to accommodate changes to the environment or the system model. 

The proposed safety filter is summarized in Algorithm \ref{alg:rlwithddps}.
It uses a receding-horizon strategy to create a new safe plan $\plan_k$ in each $k$ receding-horizon motion planning iteration.
Consider a single planning iteration $k$ (Lines \ref{ln:observe1}--\ref{ln:train_rl_agent}).
Suppose the robot has previously created a safe plan $\plan_{k-1}$ (such as staying stopped indefinitely).
At the beginning of each iteration, our framework samples the RL policy for a new action $\RLaction(k)$ (Line \ref{ln:sample_rl}). Next, it enforces safety on the action chosen by the RL agent by solving a data-driven optimal control problem as in Algorithm \ref{alg:enforce_safety}. If Algorithm \ref{alg:enforce_safety} fails to find a safe plan, it executes a failsafe maneuver instead (Lines \ref{ln:if_unsafe}--\ref{ln:catch_unsafe}). Then, we apply the first action of the safe plan to the environment, get a reward, observe the environment, and train the RL agent from the replay buffer (Lines \ref{ln:get_safe_action}--\ref{ln:train_rl_agent}). We note that the proposed algorithm can be used during both training and deployment. That is, the safety layer can operate even for an untrained policy while initializing $\policy_\theta$ with random weights.


 
 
Data-driven predictive control solves an optimization problem to find a  reachable set (represented by a zonotope) under some constraints \cite{datadriven_predictive_control}. It utilizes data-driven reachability analysis \cite{alanwar2020data_conf,alanwar2020data} under the hood to overapproximate the reachable sets of an unknown system from noisy data collected offline. In this work, we assume a linear system model.

We consider $\ntraj$ input-state trajectories of lengths $T_i \in \N$, $i = 1,\cdots,\ntraj$, with total duration $T\total = \sum_i^\ntraj T_i$.
We denote the data as $(y\idxi_k)_{k=0}^{t_i}$, $(\action\idxi_k)_{k=0}^{t_i-1}$, $i=1, \cdots, \ntraj$.
To further ease the, we rearrange the data in the following matrices:
\begin{subequations}\label{eq:data}
\begin{align}
    \statedata_- &= \left[y\idx{1}_0,\cdots,y\idx{1}_{t_1-1}, \cdots, y\idx{\ntraj}_0, \cdots,y\idx{\ntraj}_{t_\ntraj-1} \right],\\
    \statedata_+ &= \left[y\idx{1}_1,\cdots,y\idx{1}_{t_1}, \cdots, y\idx{\ntraj}_1, \cdots,y\idx{\ntraj}_{t_\ntraj} \right],\\
    \actiondata &= \left[\action\idx{1}_0,\cdots,\action\idx{1}_{t_1-1}, \cdots,\action\idx{\ntraj}_0,\cdots,\action\idx{\ntraj}_{t_\ntraj-1} \right],\label{eq:input_data}
\end{align}
\end{subequations}
and similarly for the unknown noise realizations $V_-,V_+$ and $W_-$. 
It follows directly from Assumptions \ref{assumption:noise_zonotope} and \ref{assumption:meas_noise_zonotope} that the noise realizations are within corresponding matrix noise zonotopes $V_-,V_+ \in \mathcal{M}_v$, $W_- \in \mathcal{M}_w$ and $AV_- \in \mathcal{M}_{Av}$ as shown in \cite{alanwar2020data}.

Next, we compute a set of models consistent with the collected data which is utilized to compute the data-driven reachable sets in the following lemmas. 
\begin{lemma}[\cite{datadriven_predictive_control}]
\label{th:setofAB}
Given input-output trajectories $D = \{U_-, Y\}$ of the system \eqref{eq:sys}, then 
\begin{align}
    \mathcal{M}_\Sigma = (Y_{+} - \mathcal{M}_w - \mathcal{M}_v + \mathcal{M}_{Av})\begin{bmatrix} 
    Y_- \\ U_- 
    \end{bmatrix}^\dagger
   \label{eq:zonoAB}
\end{align} 
 contains all matrices $\begin{bmatrix}A & B \end{bmatrix}$ that are consistent with the data and noise bounds. 
\end{lemma}

\begin{lemma}[\cite{datadriven_predictive_control}]
\label{th:reach_lin}
Given input-output trajectories $D =\{U_-, Y\}$ of the system in \eqref{eq:sys}, then 
\begin{align}
\hat{\mathcal{R}}_{t+1} &= \mathcal{M}_{\Sigma} (\hat{\mathcal{R}}_{t} \times \mathcal{Z}_{u,t}  ) +  \mathcal{Z}_w +  \mathcal{Z}_v -   \mathcal{Z}_{Av},\label{eq:Rkp1}
\end{align}
contains the reachable set, i.e., $\hat{\mathcal{R}}_{t+1} \supseteq \mathcal{R}_{t+1}$ where $\hat{\mathcal{R}}_{0}=\zono{y(0),0}$, and $\mathcal{Z}_{u,t}=\zono{u(t),0}$. 
\end{lemma}
To enforce safety, we must generate a plan instead of a single action. The reason behind this is that safety can't be enforced for a single action as the actuation of the system is finite (e.g., a robot needs N time steps to brake). Algorithm~\ref{alg:enforce_safety} enforces safety while being as little invasive as possible to the RL agents' choice of action. This is done by solving a data-driven optimal control problem as in Line \ref{ln:opt_problem} of Algorithm~\ref{alg:enforce_safety}. This optimization problem is solved under some constraints (avoiding obstacles). First, we apply 
data-driven reachability on the state zonotope $\hat{\mathcal{R}}_{t+k+1|t}$ in \eqref{eq:Rconst}. The computed reachable set $\hat{\mathcal{R}}_{t+k+1|t}$, which includes the output in line \eqref{eq:y_in_r}, is ensured to satisfy the output constraints $\mathcal{Y}_{t+k+1}$ in order to guarantee the safety in \eqref{eq:const_y}. The initial state in the optimization problem is set to be the same as the robot state in \eqref{eq:y0const}.

\begin{remark}
The constraints \eqref{eq:const_y} is implemented by converting both zonotopes into intervals and ensuring that 
\begin{align*}
   &\hat{\mathcal{R}}_{u,t+k+1|t}  \leq \mathcal{Y}_{u,t+k+1}, \quad \hat{\mathcal{R}}_{l,t+k+1|t}  \geq \mathcal{Y}_{l,t+k+1} 
\end{align*}
where $\hat{\mathcal{R}}_{u,t+k+1|t}$ and $\hat{\mathcal{R}}_{l,t+k+1|t}$ are the upper and lower bounds of $\hat{\mathcal{R}}_{t+k+1|t}$, and $\mathcal{Y}_{u,t+k+1}$ and $\mathcal{Y}_{l,t+k+1}$ are the upper and lower bounds of $\mathcal{Y}_{t+k+1}$. 
\end{remark}

\begin{algorithm}[t]
\caption{Enforce Safety Using Data Driven Predictive Control}
\label{alg:enforce_safety}

\KwInput{action $\RLaction(k)$, state $\vcoutput(k)$, obstacles $\statespace\obs$, initial reachable set $\reachapprox_k$, time limit $t\lbl{max}$, planning horizon $n_{plan}$, and process noise zonotope $\mathcal{Z}_w$} 
\KwOutput{Safe plan $\plan_k$, if available}
Set the reference input action $r \leftarrow \RLaction(k)$, and the obstacles free area to $\mathcal{Y}_{i+k+1}$

Solve the optimization problem (for $i=0, \dots, n_{plan}$) where the input action zonotope $\mathcal{Z}_{u,k} {\leftarrow} \zono{u_{i+k|k}, 0}$ \label{ln:opt_problem}
\begin{subequations}
\label{eq:optzonopc}
\begin{alignat}{2}
&\!\min_{u,y}        &\quad&\!\!\! \norm{u_{k|k} - r}^2 \\
&\text{s.t.} &      &\!\!\! \hat{\mathcal{R}}_{i+k+1|k} {=} \mathcal{M}_{\Sigma} (\hat{\mathcal{R}}_{i+k|k} {\times} \mathcal{Z}_{u,k}  ) {+}  \mathcal{Z}_w{+}  \mathcal{Z}_v{-} \mathcal{Z}_{Av},  \label{eq:Rconst}\\
&                  &      &\!\!\!  u_{i+k|k} \in \mathcal{U}_{i+k}, \label{eq:uconst}\\
&                  &      & \!\!\!\hat{\mathcal{R}}_{i+k+1|k} \subseteq \mathcal{Y}_{i+k+1}, \label{eq:const_y}\\
&                  &      & \!\!\!y_{i+k+1|k}\in \hat{\mathcal{R}}_{i+k+1|k}, \label{eq:y_in_r}\\
&                  &      & \!\!\!y_{k|k} = y(k) \label{eq:y0const}
\end{alignat}
\end{subequations}

\eIf{all $\reachapprox_j \cap \statespace\obs = \emptyset$ \label{ln:adjust:check_if_safe}}{
    
    \textbf{return} $\plan_k = (\action_j)_{j=k}^{\nplan}$ // found a safe plan
}{
\textbf{return} $\plan_k = \emptyset$  // No safe plan was found \label{ln:adjust:return_unsafe}
}
\end{algorithm}

We conclude this section by formalizing the safety guarantees of our proposed algorithm.

\begin{theorem}\label{thm:BRSL_is_safe}
Suppose the assumptions on the robot and environment from Section \ref{sec:pb} all hold, and, at time $k = 0$, the robot is \tbold{at safe state}.
Suppose also that, at each time $k > 0$, the robot \tbold{rolls out} a new $\plan_k$, and adjusts the plan using Algorithm \ref{alg:enforce_safety}.
Then, the robot is guaranteed to be safe at all times $k \geq 0$.
\end{theorem}
\begin{proof}
The proof can be done by induction. At the time $0$, the agent can apply a braking action $\action\brk$ to stay safe for all time.
Assume a safe plan exists at time $k \in \N$. Then, at any timestep $k + 1$, in case of failure of Algorithm \ref{alg:enforce_safety} to find a safe plan, and since the plan is always larger than the number of steps required by the robot to brake to a stop fully, the robot always has a failsafe braking maneuver that enables it to stop indefinitely; otherwise, if a new plan is found, the plan is safe for two reasons.
First, the reachability algorithm is guaranteed to contain the true reachable set of the system \cite[Theorem 2]{alanwar2020data_conf} because the noise zonotopes bound noise as in Assumptions \ref{assumption:noise_zonotope} and \ref{assumption:meas_noise_zonotope}.
Second, when adjusting an unsafe action with Algorithm \ref{alg:enforce_safety}, the algorithm will only return a safe plan in which the first action is as close as possible to the RL action of choice while enforcing the hard constraints of safety. We emphasize that in order for a plan to be deemed safe, the plan length must be larger than the number of steps needed for the robot to brake to a full stop.
\end{proof}
\section{Experiments} \label{sec:eval}
We evaluate the proposed frameworks on two motion planning tasks in simulation: A Turtlebot in Gazebo and a Quadrotor in Unreal Engine 4 (UE4). Fig. \ref{fig:sim_envs} shows our simulation environments. The simulation shows that the proposed framework shields the RL agents from unsafe actions while maximizing the objective. Next, we describe our simulation environments in detail.

\subsection{Simulation Environments Setup}
\subsubsection{Turtlebot Setup}
We use Turtlebot 3 in the Gazebo simulator. The robot has a longitudinal velocity in $[0, 0.25]$ m/s, and an angular velocity in $[-0.5, 0.5]$ rad/s. We clip the upper longitudinal velocity at $0.25$ m/s to be consistent with the real robot dynamics. The robot is also provided with wheel encoders, an IMU to estimate the robot's speed and position, and a planar lidar with 18 laser measurements is evenly spanning $180^{\circ}$. 
In each episode, the robot is spawned randomly in a safe spot. It's required to reach the goal on the map denoted by the green circle in Fig. \ref{fig:tb_env}. The goal is chosen at a random location in each episode, and the obstacles are spawned randomly in each episode as well. 

\begin{figure}[t]
\centering     
\begin{subfigure}[h][Turtlebot3 Environment]{\label{fig:tb_env}\includegraphics[width=4cm, height=4cm]{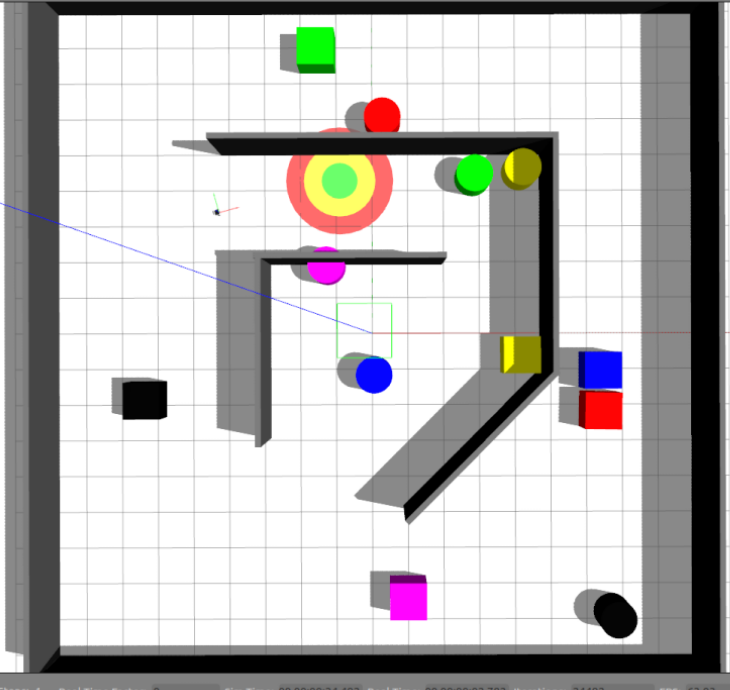}}
\end{subfigure}
\begin{subfigure}[h][Quadrotor Environment]{\label{fig:quad_env}\includegraphics[width=4cm, height=4cm]{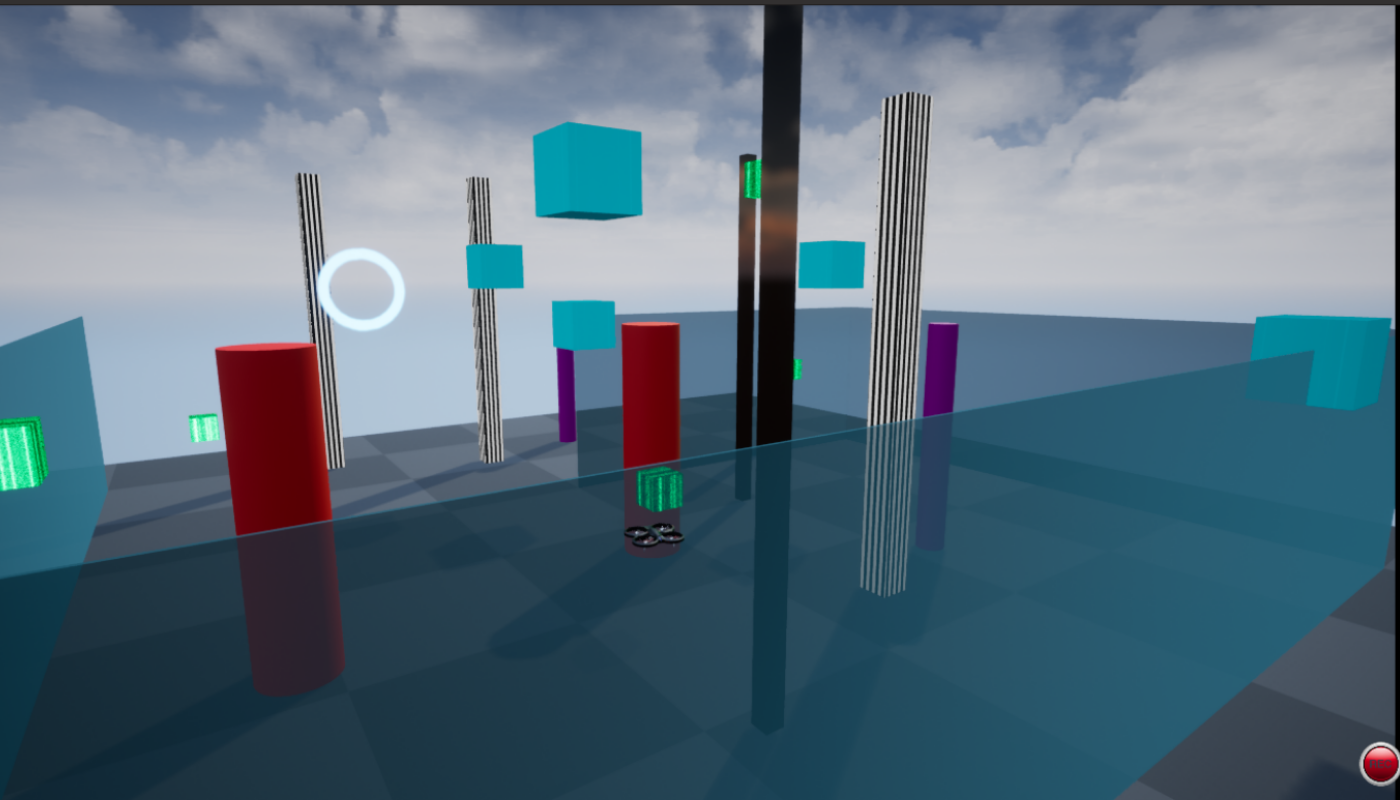}}
\end{subfigure}
\caption{Evaluation Environments.}
\label{fig:sim_envs}
\vspace{-4mm}
\end{figure}

\subsubsection{Quadrotor}
The quadrotor is simulated in Unreal Engine 4 (UE4). They have a similar dynamics with a velocity ranging from $[-5, -5, -5]$ m/s to $[5, 5, 5]$ m/s. Both are an IMU to estimate the rotors' position and velocity and are provided with a multi-channel 3D lidar. The lidar spans a range of $360^{\circ}$ horizontally and a range of $60^{\circ}$ vertically. Again, the environments' obstacles and goals are spawned randomly at the beginning of each episode. The goal of the quadrotor is to reach a goal position while avoiding the obstacles. 

\begin{figure*}
\centering     
\begin{subfigure}[h][Turtlebot3 Reward]{\label{fig:turtlebot_reward}\includegraphics[width=0.35\textwidth]{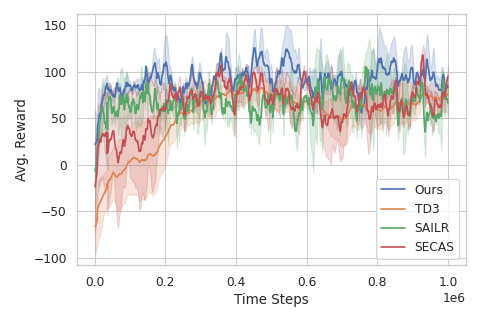}}
\end{subfigure}
\begin{subfigure}[h][Quadrotor Reward]{\label{fig:drone_reward}\includegraphics[width=0.35\textwidth]{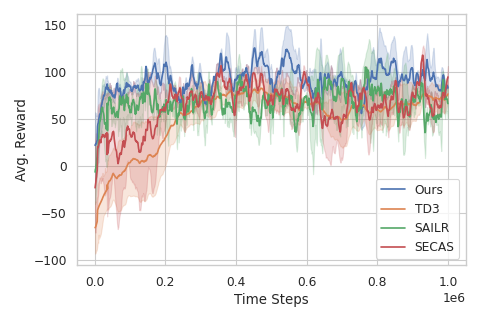}}
\end{subfigure}

\caption{Average reward over time of Our method (blue), SAILR \cite{wagener2021safe}, SECAS \cite{dalal2018safe}, and a vanilla TD3 baseline for each of our experiments.}
\label{fig:sim_results}
\vspace{-2mm}
\end{figure*}

\subsection{RL Agent Setup}

\subsubsection{Turtlebot Environment}
We use TD3 \cite{conf:TD3} as our RL agent. The state vector $\RLoutput \in \mathbb{R}^{26}$ consists of the state of the physical robot $\in \mathbb{R}^{8}$ (position and orientation) estimated from the noisy wheel encoders and IMU, and the lidar measurements $\in \mathbb{R}^{18}$.

\subsubsection{Quadrotor Setup}
Similar to the Turtlebot example, we use TD3 as our RL agent. The state vector $\RLoutput \in \mathbb{R}^{189}$ consists of the state of the physical robot $\in \mathbb{R}^{9}$ (position only) estimated from the noisy IMU, and the lidar measurements $\in \mathbb{R}^{180}$. 

\subsection{Results and Discussion}
The results are summarized in Table \ref{tbl:comparison}, and Fig. \ref{fig:sim_results}. We see that the proposed method could achieve the highest reward and outperform all the other methods. As for the safety constraints, we see from Table \ref{tbl:comparison} that both the proposed method and SECAS could impose hard safety constraints (i.e., having $0$ collision rate). However, both SAILR and TD3 agents couldn't have full constraint satisfaction. Regarding the goal rate and the robot speed, our proposed method outperforms all the other methods without being overly conservative. Finally, regarding computation time, we see that the vanilla TD3 RL agent outperformed all the other methods due to the absence of any safety modules in the framework. However, we note that the proposed method is real-time even when dealing with large data matrices since the calculations can be done efficiently on GPUs. 


\begin{table*}[ht]
\caption{Numerical experiment results (best values in bold)}
\label{tbl:comparison}
\vspace{-3mm}
\centering
\resizebox{1.98\columnwidth}{!}{\begin{tabular}{ c|c|c|c|c||c|c|c|c|} 
 & \multicolumn{4}{c||}{Turtlebot} & \multicolumn{4}{c}{Quadrotor} \\
 & ours & SAILR & \tbold{SECAS} & Baseline  & ours & SAILR & \tbold{SECAS}& Baseline  \\
 \hline
 Goal Rate [\%] &
    \textbf{55} & 50 & 53 & 38  &
    \textbf{68} & 60 & 61 & 46 \\
 Collision Rate [\%] &
    \textbf{0.0} & 4.9 & \tbold{0.0} & 17 
    &\textbf{0.0} & 5.6 & \tbold{0.0} & 22\\
 Mean/Max Speed [m/s] & 
    .06 / \textbf{0.25} & .04 / 0.17 & \tbold{.06} / 0.21 &\textbf{.08} / \textbf{.25} & 
    3.5 / \textbf{8.6} &  3.2 / 8.2 & 3.2 / 8.3 & \textbf{3.7} / \textbf{8.6}\\
 Mean Reward &
    \textbf{84} & 69 & 66 & 58 &
    \textbf{86} & 78 & 72 & 63 \\
 Mean $\pm$ Std. Dev.  &
    \multirow{2}{*}{80.41 $\pm$ 15.0} & \multirow{2}{*}{29.8 $\pm$ 11.8} & \tbold{\multirow{2}{*}{21.4 $\pm$ 11.8}} & \multirow{2}{*}{\textbf{8.0} $\pm$ 7.19} &
    \multirow{2}{*}{80.9 $\pm$ 17.27} & \multirow{2}{*}{42.22 $\pm$ 17.58} & \tbold{\multirow{2}{*}{32.7 $\pm$ 28.05}} & \multirow{2}{*}{\textbf{12.7} $\pm$ 7.5} \\
    Computation Time [ms] &&&&&&&& \\
 \hline
    \end{tabular}}
\vspace{-4mm}
\end{table*}
\section{Conclusion}\label{sec:con}

This paper proposes a framework for safe reinforcement learning. The proposed framework leverages data-driven methods to provide safety guarantees. The framework was evaluated on a path planning task for Turtlebot 3 and a quadrotor, where we proved mathematically and experimentally that the proposed method could indeed provide safety under reasonable assumptions. 
For future work, we will explore continuous-time settings, reducing the conservativeness of our reachability analysis, and minimizing the amount of data needed to guarantee safety.

\section*{Acknowledgement}

This paper has received funding from the European Union's Horizon 2020 research and innovation programme under grant agreement No. 830927.

\bibliographystyle{IEEEtran}
\bibliography{ref} 


\end{document}